\documentclass[letterpaper]{article} 
\usepackage{aaai2026}  
\usepackage{times}  
\usepackage{helvet}  
\usepackage{courier}  
\usepackage[hyphens]{url}  
\usepackage{graphicx} 
\usepackage{natbib}  
\usepackage{caption} 
\usepackage{amsfonts}
\usepackage{subfigure}
\usepackage{amsmath}
\usepackage{bm}
\usepackage[ruled, vlined, linesnumbered]{algorithm2e}
\usepackage{booktabs}
\usepackage{xspace}
\usepackage{color}
\usepackage[switch]{lineno}
\usepackage{amsthm}
\makeatletter
\DeclareRobustCommand\onedot{\futurelet\@let@token\@onedot}
\def\@onedot{\ifx\@let@token.\else.\null\fi\xspace}
 
\def\ie{\emph{i.e}\onedot}

\newtheorem{theorem}{Theorem}

\makeatother

\makeatletter
\renewenvironment{align*}{%
  \linenomath 
  \start@align\@ne\st@rredtrue\m@ne
}{%
  \endalign
  \endlinenomath 
}
\makeatother

\title{Deep (Predictive) Discounted Counterfactual Regret Minimization}

\author{
    Hang Xu$^{1, 2}$,
    Kai Li$^{1, 2,}$\thanks{Corresponding author.},
    Haobo Fu$^{6}$,
    Qiang Fu$^{6}$,
    Junliang Xing$^{5}$,
    Jian Cheng$^{1,3,4}$\\
}
\affiliations {
$^1$C$^2$DL, Institute of Automation, Chinese Academy of Sciences\\
$^2$School of Artificial Intelligence, University of Chinese Academy of Sciences\\
$^3$AiRiA~~~$^4$Maicro.ai~~~$^5$Tsinghua University~~~$^6$Tencent AI Lab \\
\{xuhang2020, kai.li, jian.cheng\}@ia.ac.cn, 
\{haobofu, leonfu\}@tencent.com, jlxing@tsinghua.edu.cn
}

\begin{document}

\maketitle

\begin{abstract}
    Counterfactual regret minimization (CFR) is a family of algorithms for effectively solving imperfect-information games. 
    To enhance CFR's applicability in large games, researchers use neural networks to approximate its behavior. 
    However, existing methods are mainly based on vanilla CFR and struggle to effectively integrate more advanced CFR variants.
    In this work, we propose an efficient model-free neural CFR algorithm, overcoming the limitations of existing methods in approximating advanced CFR variants. At each iteration, it collects variance-reduced sampled advantages based on a value network, fits cumulative advantages by bootstrapping, and applies discounting and clipping operations to simulate the update mechanisms of advanced CFR variants.
    Experimental results show that, compared with model-free neural algorithms, it exhibits faster convergence in typical imperfect-information games and demonstrates stronger adversarial performance in a large poker game.
\end{abstract}

\begin{links}
    \link{Code}{https://github.com/rpSebastian/DeepPDCFR}
\end{links}

\section{Introduction}

Imperfect-information games (IIGs) serve as a foundational framework for modeling strategic interactions among multiple players where certain information remains hidden. 
Addressing these games poses significant challenges, as it requires reasoning under uncertainty about opponents' private information. 
Such hidden information is pervasive in real-world scenarios, such as negotiation~\cite{gratch2016misrepresentation}, security~\cite{lisy2016counterfactual}, medical treatment~\cite{sandholm2015steering}, and recreational games~\cite{brown2019superhuman}, making research on IIGs theoretically and practically crucial.  
The primary goal in solving IIGs is to compute an (approximate) Nash equilibrium (NE)~\cite{nash1950equilibrium}---a strategy profile where no player can gain by unilaterally altering its strategy.

Similar to most research on solving IIGs, we focus on learning an NE in two-player zero-sum IIGs. We also assume the model-free setting, where the algorithm does not have an exact simulator of the game and only samples episodes from the game. When the perfect game model is available, the family of counterfactual regret minimization (CFR) algorithms~\cite{zinkevich_regret_2007,tammelin_solving_2014,brown_solving_2019,farina2021faster} is one of the most successful approaches for computing an NE, which iteratively minimizes the cumulative counterfactual regrets of both players so that the average strategy profile converges to an NE in two-player zero-sum IIGs. Due to its robust theoretical foundation and strong empirical performance, CFR and its variants have driven several significant advancements in this field~\cite{bowling_heads-up_2015,moravvcik2017deepstack,brown2018superhuman,brown2019superhuman}. 

When lacking a game model, outcome-sampling Monte Carlo CFR (OS-MCCFR)~\cite{lanctot2009monte} has been proposed to approximate the counterfactual regrets in each iteration by sampling episodes from the game. To further scale the algorithm to large IIGs, many novel neural CFR variants have been developed. OS-DeepCFR~\cite{brown2019deep} employs function approximation with deep neural networks to approximate the cumulative counterfactual regrets instead of tabular storage. DREAM~\cite{steinberger2020dream}, which is built upon variance-reduced MCCFR~\cite{schmid2019variance}, employs a learned value function as a baseline to reduce the high variance in estimating cumulative counterfactual regrets. ESCHER~\cite{ESCHER} uses a history value function and a fixed sampling strategy for the updating player to avoid using importance sampling.

Although these neural approaches have greatly accelerated CFR in large IIGs, they primarily focus on approximating the behavior of vanilla CFR or the variant LinearCFR~\cite{brown_solving_2019}. Besides, they rely on a large replay buffer to store experiences, which is used to refit the cumulative counterfactual regrets each iteration. Recent advancements in tabular CFR have demonstrated that novel CFR variants can achieve significantly faster convergence compared to vanilla CFR and LinearCFR. To reduce the cost of picking wrong actions, CFR+~\cite{tammelin_solving_2014} clips negative cumulative counterfactual regrets in each iteration, Discounted CFR (DCFR)~\cite{farina2021faster} discounts the cumulative counterfactual regrets in each iteration, and DCFR+~\cite{AutoCFR} combines the key insights of CFR+ and DCFR to achieve faster convergence. Predictive CFR+ (PCFR+)~\cite{farina2021faster} leverages the predictability of the counterfactual regrets in each iteration to accelerate the convergence speed. PDCFR+~\cite{PDCFRPlus} integrates PCFR+ and DCFR in a principled manner, showcasing faster convergence in non-poker IIGs. These tabular variants provide a promising avenue for neural CFR to further unlock its potential in achieving faster convergence.

To this end, we propose novel model-free neural CFR variants, Variance Reduction Deep DCFR+ (VR-DeepDCFR+) and Variance Reduction Deep PDCFR+ (VR-DeepPDCFR+). These algorithms employ deep neural networks to approximate the behavior of DCFR+ and PDCFR+, respectively, while leveraging learned baseline functions to mitigate the high variance caused by episode sampling. 

Approximating advanced CFR variants presents a significant challenge, as these tabular CFR variants update cumulative counterfactual regrets by bootstrapping from the previous iteration. This makes it impossible to directly approximate them using samples from all iterations stored in the replay buffer, as is typically done when approximating vanilla CFR or LinearCFR. Furthermore, approximating the counterfactual regrets in each iteration is particularly challenging since their computation depends on expectation values weighted by the opponent's reach probabilities. These unnormalized values introduce substantial difficulties for neural networks to approximate effectively. Advantages, however, have been shown to be effectively approximated and generalized using neural networks~\cite{schulman2017proximal}, and can be interpreted as a form of weighted counterfactual regrets~\cite{srinivasan2018actor}. Therefore, we directly approximate cumulative advantages during each iteration by using the samples collected in the current iteration while bootstrapping from the results of the previous iteration.

Moreover, under the model-free setting, the algorithm samples only a single action per state, resulting in high variance in the collected samples during each iteration. 
To mitigate this issue, we introduce an auxiliary value network inspired by DREAM~\cite{steinberger2020dream} to reduce the variance induced by episode sampling. 
Building upon the approximated cumulative advantages, we integrate the novel tabular variants DCFR+ and PDCFR+ into neural CFR. For DCFR+, we apply discounting and clipping to the cumulative advantages at each iteration, enabling efficient action reuse and controlling the potential overgrowth of cumulative advantages over iterations. For PDCFR+, we further introduce an additional network to fit the advantages at each iteration, which is then used to predict the next advantages and compute the new strategy. Experimental results demonstrate that these algorithms achieve competitive performance compared to other model-free neural methods.

\section{Preliminaries}
\subsection{Notations}
Extensive-form games~\cite{osborne1994course} provide a tree-based formalism widely used to describe IIGs. These games involve a finite set $\mathcal{N} = \{1, 2, \ldots, N\}$ of \textbf{players}, along with a special player $c$ called \textbf{chance} following a fixed, known stochastic strategy. A \textbf{history} $h$ represents a sequence of all actions taken by players, including any private information available only to specific players. The set of all possible histories forms $\mathcal{H}$, while $\mathcal{Z} \subseteq \mathcal{H}$ denotes \textbf{terminal histories} where no further actions are possible. The set of terminal histories that can be reached from a history $h$ is denoted by $\mathcal{Z}[h] = \left\{ z \in \mathcal{Z}: h \sqsubseteq z\right\}$, where the relationship $g\sqsubseteq  h$ indicates that $g$ is equal to or a \textbf{prefix} of $h$. At any given history $h$, players choose from the \textbf{actions} available, represented by $\mathcal{A}(h) = \{a: ha \in \mathcal{H}\}$. The player making the decision at history $h$ is denoted by $\mathcal{P}(h)$. Each player $i \in \mathcal{N}$ is associated with a \textbf{utility function} $u_i(z): \mathcal{Z} \rightarrow \mathbb{R}$, which assigns a  utility to each terminal history $z \in \mathcal{Z}$. 

In IIGs, the lack of information is modeled using \textbf{information sets} $\mathcal{I}_i$ for each player $i \in \mathcal{N}$, where $h, h' \in \mathcal{I} $ indicates that player $i$ cannot distinguish between them based on the information available. For instance, in poker, histories in the same information set differ only in opponents' private cards. Hence, $\mathcal{A}(I) = \mathcal{A}(h)$ and $\mathcal{P}(I) = \mathcal{P}(h)$ for any $h \in I$. The set of all terminal histories that can be reached from an information set $I$ is denoted by $\mathcal{Z}[I] = \bigcup_{h \in I} \mathcal{Z}[h]$, and $z[I]$ is the unique history $h\in I$ such that $h\sqsubseteq z$. 

A \textbf{strategy} $\sigma_i(I)$ assigns a probability distribution over actions $\mathcal{A}(I)$ available to player $i$ in information set $I$, where $\sigma_i(I, a)$ represents the probability of player $i$ choosing action $a$ in $I$. Similarly, the strategies must be consistent across all histories in an information set. Thus, for any $h_1, h_2 \in I$, we have $\sigma_i(I) = \sigma_i(h_1) = \sigma_i(h_2)$. A \textbf{strategy profile} $\sigma = \{\sigma_i \mid \sigma_i \in \Sigma_i, i \in \mathcal{N}\}$ specifies the strategies for all players, where $\Sigma_i$ denotes the set of all possible strategies for player $i$, and $\sigma_{-i}$ refers to the strategies of other players.

The \textbf{history reach probability} $\pi^\sigma(h)$ is the joint probability of reaching history $h$ under strategy profile $\sigma$, computed as $\pi^\sigma(h)=\prod_{h^\prime a \sqsubseteq h}\sigma_{\mathcal{P}(h')}(h',a)$. It factorizes as $\pi^\sigma(h)=\pi^\sigma_i(h)\pi^{\sigma}_{-i}(h)$, where $\pi^\sigma_i(h)$ is player $i$'s contribution, and $\pi^{\sigma}_{-i}(h)$ is the contributions of all other players. The \textbf{information set reach probability} $\pi^\sigma(I)$ is defined as $\pi^\sigma(I)=\sum_{h \in I}^{}\pi^\sigma(h)$, and the \textbf{interval information set reach probability} from $h$ to $h'$ is defined as $\pi^\sigma(h,h')=\pi^\sigma(h) / \pi^\sigma(h')$ if $h' \sqsubseteq h$.  $\pi_i^\sigma(I), \pi^\sigma_{-i}(I), \pi^\sigma_{i}(h, h'), \pi^\sigma_{-i}(h, h')$ are defined similarly.

The \textbf{expected utility} $u_i(\sigma_i, \sigma_{-i})$ for player $i$ denotes her utility when playing $\sigma_i$ against opponents' strategy $\sigma_{-i}$. Formally, $u_i(\sigma_i, \sigma_{-i}) = \sum_{z \in \mathcal{Z}}^{}\pi^\sigma(z)u_i(z)$.
At the level of an information set $I$, the expected utility for taking action $a$ is 
\begin{align*}
    u_i^\sigma(I, a) =\frac{\sum_{h \in I}^{}\pi^\sigma(h)\sum_{z \in \mathcal{Z}[ha]}^{}\pi^{\sigma}(ha, z)u_i(z) }{\sum_{h \in I}{\pi^\sigma(h)}}, 
\end{align*}
and for the whole information set: $u_i^\sigma(I) = \sum_{a\in \mathcal{A}(I)}^{}\sigma_i(I,a)u_i^\sigma(I, a)$. The \textbf{advantage} $A_i^\sigma(I, a) = u_i^\sigma(I, a) - u_i^\sigma(I)$ then quantifies the utility gain of playing action $a$ instead of the current strategy $\sigma_i(I, a)$.

\subsection{Best Response and Nash Equilibrium}
The \textbf{best response} to $\sigma_{-i}$ is any strategy $\text{BR}(\sigma_{-i})$ that maximizes the expected utility, satisfying $u_i(\text{BR}(\sigma_{-i}), \sigma_{-i})=\max_{\sigma'_{i}\in\Sigma_i}u_i(\sigma_i', \sigma_{-i})$. An \textbf{NE} is a strategy profile $\sigma^*=\left(\sigma_i^*, \sigma_{-i}^*\right)$ where each player plays a best response to the others, ensuring $\forall i\in\mathcal{N}, u_i(\sigma_i^*, \sigma_{-i}^*) = \max_{\sigma_i'\in \Sigma_i}u_i(\sigma_i', \sigma_{-i}^*)$. The \textbf{exploitability} of a strategy $\sigma_i$ is defined as $e_i(\sigma_i) = u_i(\sigma_i^*, \sigma_{-i}^*) - u_i(\sigma_i, \text{BR}(\sigma_i))$. In an $\boldsymbol{\epsilon}$\textbf{-NE}, no player's exploitability exceeds $\epsilon$. The exploitability of a strategy profile $\sigma$, given by $e(\sigma)=\sum_{i\in\mathcal{N}}e_i(\sigma_i)/|\mathcal{N}|$, represents the approximation error relative to the NE.

\subsection{Counterfactual Regret Minimization}
Counterfactual Regret Minimization (CFR)~\cite{zinkevich_regret_2007} frequently uses \textbf{counterfactual value} $v_i^\sigma(I, a)$, which is the expected utility of an action $a$ at an information set $I\in \mathcal{I}_i$ for player $i$, weighted by the probability that $i$ reaches $I$. Formally, $v_i^\sigma(I, a) = \sum_{z \in Z[I]}^{}\pi_{-i}^\sigma(z[I])\pi^\sigma(z[I]a, z) u_i(z)$, and $v_i^\sigma(I) = \sum_{a \in \mathcal{A(I)}}^{}\sigma_i(I, a)v_i^\sigma(I)$.
Starting from a uniform strategy $\sigma^1$, CFR traverses the game tree in each iteration $t$ to computes the \textbf{instantaneous counterfactual regret} $r^t_i(I, a)=v_i^{\sigma^t}(I, a)-v_i^{\sigma^t}(I)$, which relates to the advantages as $r_i^t(I, a)=A_i^{\sigma^t}(I, a)\pi_{-i}^{\sigma^t}(I)$~\cite{srinivasan2018actor}, accumulates it into \textbf{cumulative counterfactual regret} $R^t_i(I, a)=\sum_{k=1}^{t}r_i^{k}(I, a)$, and updates strategies by regret-matching~\cite{hart2000simple}: $\sigma_i^{t+1}(I, a)=\frac{\max(0, R_i^{t}(I, a))}{\sum_{a^{\prime} \in \mathcal{A}(I)}\max(0, R_{i}^{t}\left(I, a^{\prime}\right))}$, using the uniform strategy when  all regrets are non-positive.
The \textbf{average strategy} $\bar{\sigma}^t$ converges to an NE and is computed via \textbf{cumulative strategy} $C_i^t(I, a)$:
\begin{align*}
	\resizebox{0.93\linewidth}{!}{$
	\!\!\!\!\!C_i^{t}(I, a)\!\!=\!\!\sum_{k=1}^{t}\left(\pi_{i}^{\sigma^{k}}(I) \sigma_{i}^{k}(I, a)\right), \bar{\sigma}_i^t(I, a)\!\!=\!\!\frac{C_i^t(I,a)}{\sum_{a'\in\mathcal{A}(I)}C_i^t(I, a')}.
    $}
\end{align*}

\begin{algorithm*}[t]
	\KwIn{total iterations $T$, traversal times $K$, parameters $\alpha, \gamma$, exploration coefficient $\epsilon$.}
    \SetKwFunction{FMain}{Traverse}
    \SetKwProg{Fn}{Function}{:}{}
    Initialize each player's cumulative advantage network $R(I, a\mid\theta^0_i)$ with parameters $\theta^0_i$\;
    Initialize each player's instantaneous advantage network $r(I, a\mid\phi^0_i)$ with parameters $\phi^0_i$\;
    Initialize history value network $Q(h, a \mid w^0)$ and history value buffer $\mathcal{B}_Q$\;
    Initialize reservoir-sampled strategy buffer $\mathcal{B}_{\Pi}$ and each player's advantage buffer $\mathcal{B}_{V, i}$\;
    \For{CFR iteration $t=1$ to $T$}{
        clear each player's advantage buffer $\mathcal{B}_{V, i}$\;
        \ForEach{player $i$}{
            \For{traversal $k=1$ to $K$}{
                \FMain($\emptyset$, $i$, $\mathcal{B}_{V, i}$, $\mathcal{B}_{\Pi}$, $\mathcal{B}_Q$, $t$), $\triangleright$(Algorithm~\ref{algo2})\;
            }
            Train $\theta_i^t$ on loss $\mathcal{L}(\theta^t_i) = \mathbb{E}_{(I, \bar{r} )\sim \mathcal{B}_{V, i}}\left[\sum_{a \in \mathcal{A}(I)}^{}\left(\max(R(I, a \mid \theta^{t-1}_{i} ), 0)\frac{(t-1)^\alpha}{(t-1)^\alpha + 1} +\bar{r}(I, a)-R(I, a\mid\theta^t)\right)^2\right]$\;
            For VR-DeepPDCFR+, Train $\phi_i^t$ on loss $\mathcal{L}(\phi_i^t) = \mathbb{E}_{(I, \bar{r} )\sim \mathcal{B}_{V, i}}\left[\sum_{a \in \mathcal{A}(I)}^{}\left(\bar{r}(I, a)-r(I, a\mid\phi_i^t)\right)^2\right]$\;
            Train $w^t$ on loss $\mathcal{L}(\omega^t) = \mathbb{E}_{(t, h, a, \hat{u}, h', I', i)\sim \mathcal{B}_{Q}}\left[\left(\hat{u} + \sum_{a'\in \mathcal{A}(h')}^{}\sigma_i^{t+1}(I', a')Q(h', a' \mid \omega') - Q(h, a \mid \omega^{t}) \right)^2\right]$\;
        }
        }
    Train $\psi$ on loss $\mathcal{L}(\psi) = \mathbb{E}_{(I, t, \sigma^{t})\sim \mathcal{B}_{\Pi}}\left[\left(\frac{t}{T}\right)^\gamma \sum_{a\in \mathcal{A}(I)}\left(\sigma^{t}(I, a) - \Pi(I, a \mid \psi)\right)^2\right]$\;
    \KwOut{The average strategy network $\Pi(I, a\mid \psi)$.}
	\caption{Training procedures for VR-DeepDCFR+ and VR-DeepPDCFR+}
	\label{algo1}
\end{algorithm*}

\subsection{Tabular CFR Variants}
Since the birth of CFR, numerous variants have been proposed to accelerate convergence. 
\textbf{CFR+}~\cite{tammelin_solving_2014,bowling_heads-up_2015} improves convergence by: (1) clipping negative cumulative counterfactual regrets: $R^t_i(I, a) = \max(R_i^{t-1}(I, a) + r_i^t(I, a), 0)$; (2) using a linear weighted average strategy: $C^t_i(I, a) = C_i^{t-1}(I, a) + t\pi_i^{\sigma^t}(I)\sigma_i^t(I, a)$; (3) applying alternating updates.
\textbf{DCFR}~\cite{brown_solving_2019} further introduces discounting:
\begin{align*}
	&\resizebox{1\linewidth}{!}{$
    R_i^{t}(I, a) = \left\{\begin{array}{cl}
		R_i^{t\small{-}1}(I, a)\frac{(t\small{-}1)^\alpha}{(t\small{-}1)^{\alpha}\small{+}1}\small{+}r^{t}_i(I, a), & \text{\small{if}} \;R_i^{t\small{-}1}(I, a) \small{>} 0 \\
		R_i^{t\small{-}1}(I, a)\frac{(t\small{-}1)^\beta}{(t\small{-}1)^{\beta}\small{+}1} \small{+} r_i^{t}(I, a), & \text{otherwise,}
	\end{array}\right.
    $}\\
	&C_i^t(I, a) = C_i^{t-1}(I, a) \left(\frac{t-1}{t}\right)^\gamma+\pi_i^{\sigma^t}(I)\sigma^t_i(I, a).
\end{align*}
\textbf{LinearCFR} is a special case of DCFR, with updates $R^t_i(I, a) = R_i^{t-1}(I, a) + tr_i^t(I, a), C^t_i(I, a) = C_i^{t-1}(I, a) + t\pi_i^{\sigma^t}(I)\sigma_i^t(I, a)$.
\textbf{DCFR+}~\cite{AutoCFR,PDCFRPlus} integrates CFR+ and DCFR: $R_i^t(I, a) = \max\left(R_i^{t-1}(I, a) \frac{(t-1)^\alpha}{(t-1)^\alpha + 1} + r_i^t(I,a), 0\right)$. \textbf{PCFR+}~\cite{farina2021faster} follows CFR+ for updating cumulative counterfactual regrets and predicts the next iteration's cumulative counterfactual regrets as $\widetilde{R}_i^{t+1}(I, a) = \max(R_i^t(I, a) + \widetilde{r}_i^{t+1}(I, a), 0)$, where the prediction of instantaneous counterfactual regrets $\tilde{r}_i^{t+1}(I, a)$ is assumed to change slowly and is set to $r_i^t(I, a)$. Then PCFR+ uses $\tilde{R}_i^{t+1}(I, a)$ in regret-matching to compute the next strategy. \textbf{PDCFR+}~\cite{PDCFRPlus} updates cumulative counterfactual regrets like DCFR+ and predicts the next iteration's cumulative counterfactual regrets as $\widetilde{R}_i^{t+1}(I, a) = \max(R_i^t(I, a)\frac{t^\alpha}{t^\alpha + 1} + \widetilde{r}_i^{t+1}(I, a), 0)$. For the cumulative strategy, DCFR+, PCFR+, and PDCFR+ follow the same formula as DCFR but differ in their choices of $\gamma$. 
\subsection{Monte Carlo CFR}
The tabular CFR variants improve convergence but require full game tree traversals and tabular storage, which are infeasible in large games. 
To reduce time complexity, \textbf{Monte Carlo CFR} (MCCFR)~\cite{lanctot2009monte} estimates instantaneous counterfactual regrets by sampling portions of the game tree. 
MCCFR includes \textbf{external sampling} (ES), which explores all actions for one player while sampling actions for others, and \textbf{outcome sampling} (OS), which samples actions for all players along a single episode. This work focuses on the model-free variant OS-MCCFR, which learns directly from sampled episodes without requiring a perfect simulator.
In iteration $t$, episodes are sampled using a \textbf{sampling strategy} $\xi^t$, defined as:
\begin{align*}
\resizebox{1\linewidth}{!}{$
\xi^t_i (I, a)= \epsilon \frac{1}{\left\vert \mathcal{A}(I) \right\vert} + (1 - \epsilon)\sigma_i^t(I, a), \;\xi_{-i}^t(I, a) = \sigma_{-i}^t(I, a).
$}
\end{align*}
The \textbf{sampled counterfactual value} $\hat{v}_i^{\sigma^t}(I, a \mid z)$ uses importance sampling to ensure the unbiasedness property:
\begin{align*}
\resizebox{1\linewidth}{!}{$
\hat{v}_i^{\sigma^t}(I, a \mid z) = \frac{\pi_{-i}^{\sigma^t}(z[I])\pi^{\sigma^t}(z[I]a, z)u_i(z)}{\pi^{\xi^t}(z)} = \frac{\pi^{\sigma^t}(z[I]a, z)u_i(z)}{\pi_i^{\xi^t}(z[I])\pi^{\xi^t}(z[I], z)}.
$}
\end{align*}
The \textbf{sampled instantaneous counterfactual regrets} is  $\hat{r}_i^t(I, a)=\hat{v}_i^{\sigma^t}(I, a \mid z) - \hat{v}_i^{\sigma^t}(I \mid z)$, where $\hat{v}_i^{\sigma^t}(I \mid z) = \sum_{a \in \mathcal{A}(I)}^{}\sigma_i^t(I, a)\hat{v}_i^{\sigma^t}(I, a \mid z)$. They are unbiased estimators of $r_i^t(I, a)$. Similarly, the \textbf{sampled strategy} $\hat{\sigma}_i^{t}(I, a \mid z)$ is defined as $\sigma_i^t(I, a)\pi^{\sigma^t}_{i}(I) / \pi^{\xi^t}(z) $.

\subsection{DeepCFR}
\textbf{DeepCFR}~\cite{brown2019deep} uses neural networks to approximate CFR, avoiding tabular storage of cumulative counterfactual regrets and cumulative strategies. In each iteration, it performs $K$ partial traversals, stores the sampled instantaneous counterfactual regrets in a reservoir buffer, and trains a network from scratch to predict cumulative counterfactual regrets. The expectation value of action $a$ in an information set $I$ after $T$ iterations is $R_i^T(I, a)/\sum_{t=1}^{T}\pi^{\xi^t}(I)$, where the denominator accounts for sampling bias canceled out during regret matching.
A second buffer stores strategies across iterations to approximate the average strategy. For the sampling strategy, DeepCFR adopts ES for better performance but relies on a perfect simulator, while remaining compatible with OS. Since LinearCFR's update can be rewritten as $R_i^t(I, a)=\sum_{t=1}^{T}tr_i^t(I, a)$, DeepCFR can approximate LinearCFR by weighing samples in iteration $t$ by $t$, yielding improved performance. 

\begin{algorithm*}[t]
    \SetKwFunction{FMain}{Traverse}
    \SetKwProg{Fn}{Function}{:}{}
    \Fn{\FMain{$h$, $i$, $\mathcal{B}_{V, i}$, $\mathcal{B}_{\Pi}$, $\mathcal{B}_Q$, $t$}}{
        \KwIn{History $h$, traversing player $i$, advantage buffer $\mathcal{B}_{V, i}$, strategy buffer $\mathcal{B}_{\Pi}$, history value buffer $\mathcal{B}_Q$, iteration $t$.}
        \uIf{$h$ is terminal}{
            \KwRet the utility of player $i$\;
        }
        For VR-DeepDCFR+, compute strategy $\sigma^t_{}(I, a)$ from $R(I, a \mid \theta^{t-1}_{\mathcal{P}(h)})$ using regret matching\;
        For VR-DeepPDCFR+, compute strategy $\sigma^t(I, a)$ from the predicted cumulative advantages          
            $\max\left(R(I, a \mid \theta^{t-1}_{\mathcal{P}(h)}), 0\right) \frac{(t-1)^\alpha}{(t-1)^\alpha + 1} + r(I, a\mid \phi_{\mathcal{P}(h)}^{t-1})$ using regret matching\;
        \For{$a \in \mathcal{A}(h)$}{
            $\xi^t(I, a) \leftarrow \epsilon \frac{1}{|\mathcal{A}(h)|} + (1-\epsilon)\sigma^t(I, a)$ if $\mathcal{P}(h)=i$ else $\sigma^t(I,a)$ \;
        }
        $\hat{a} \sim \xi^t(I)$, $h'$$\leftarrow$ $h\hat{a}$\;
        \While{$\mathcal{P}(h')$ is chance}{
            $a \sim \sigma(h'), h' \leftarrow h'a$\;
        }
        $\bar{v}(I' \mid z) \leftarrow$ \FMain($h', i, \mathcal{B}_{V, i}, \mathcal{B}_{\Pi}, \mathcal{B}_Q, t$)\; 
        \For{$a \in \mathcal{A}(h)$}{
            $\bar{v}(I, a \mid z) \leftarrow$ $Q_i(h, a \mid w^{t-1} ) + \frac{\bar{v}(I' \mid z) - Q_i(h, a \mid w^{t-1})}{\xi^t(I, a)}$ if $a=\hat{a}$ else $Q_{i}(h, a \mid w^{t-1})$
        }
        
        $\bar{v}(I \mid z) \leftarrow$ $\sum_{a\in \mathcal{A}(h)}^{}\sigma^{t}(I, a) \bar{v}(I, a \mid z)$\;
        \uIf{$\mathcal{P}(h)=i$}{
            \For{$a \in \mathcal{A}(h)$}{
                $\bar{r}(I, a) \leftarrow \bar{v}(I, a) - \sum_{a' \in \mathcal{A}(I)}^{}\sigma^t(I, a')\bar{v}(I, a')$\;
            }
            Insert $(I, \bar{r})$ into the advantage buffer $\mathcal{B}_{V, i}$\;
        }
        \Else{
            Insert $(I, t, \sigma^t(I))$ into the strategy buffer $\mathcal{B}_{\Pi}$\;
        }
        Insert $(t, h, \hat{a}, \hat{u}, h', I', i)$ into the history value buffer $\mathcal{B}_Q$\;
        \KwRet $\bar{v}(I \mid z)$\;
        
  }
\caption{CFR Traversal with Outcome Sampling for VR-DeepDCFR+ and VR-DeepPDCFR+}
\label{algo2}
\end{algorithm*}

\section{Related Work}

Learning an NE in IIGs by combining deep reinforcement learning and game theory algorithms has gained considerable attention in recent years.
These approaches can generally be divided into three main categories. (1) Policy-Space Response Oracle (PSRO)~\cite{lanctot2017unified} and its variants maintain a population of strategies and iteratively compute the best response to a meta-strategy. Neural Fictitious Self-Play (NFSP)~\cite{heinrich2016deep} is a special case of PSRO, where the meta-strategy is the uniform distribution over past strategies. While these methods are effective and scalable, they rely on computationally expensive approximate best response calculations, and their convergence speed is often related to the size of the strategy space.
(2) Numerous neural CFR methods have been developed to approximate the behavior of tabular CFR~\cite{brown2019deep,lidouble,gruslys2020advantage,steinberger2020dream,liu2022model,meng2023efficient,ESCHER}. Our approach builds on this line of work by approximating novel tabular CFR variants to further accelerate convergence. (3) Another research direction involves modifying policy gradient algorithms to enable convergence to an NE~\cite{srinivasan2018actor,lockhart2019computing,hennes2020neural,fu2021actor}. However, their performance is sensitive to hyperparameters.

\section{Deep (Predictive) Discounted CFR}

\subsection{Fitting Cumulative Advantages by Bootstrapping}

DeepCFR can naturally integrate with LinearCFR but struggles to approximate the behaviors of more advanced CFR variants like DCFR+ and PDCFR+, which rely on cumulative counterfactual regrets from the previous iteration. A straightforward approach involves approximating instantaneous counterfactual regrets at each iteration and bootstrapping on the estimated cumulative counterfactual regrets from the prior iteration. However, counterfactual values are expected utilities weighted by opponents’ reach probabilities, which diminish significantly over long episodes. These reach probabilities vary widely across information sets, making it challenging for networks to effectively learn values across diverse orders of magnitude~\cite{van2016learning}. Furthermore, as demonstrated in Theorem~\ref{thm:1} (with all proofs provided in Appendix~\ref{app:proofs}), the expected estimation are scaled by sampling reach probabilities. This results in the expectation of the estimated cumulative counterfactual regrets taking the form $\sum_{t=1}^{T}\left[r_i^t(I, a)/\pi^{\xi^t}(I)\right]$. Since denominators change across iterations, it leads to deviation from CFR's behavior. 

\begin{theorem}
By using outcome sampling to collect data $(I, \hat{r}_i^t(I))$ into a buffer $\mathcal{B}_i$ for player $i$ in iteration $t$, and training a neural network $r(I, a \mid \phi^t_i)$ on loss $\mathcal{L}(\phi_i^t)=\mathbb{E}_{(I, \hat{r}_i^t(I))\sim \mathcal{B}_i}\left[\sum_{a\in \mathcal{A}(I)}^{}\left(\hat{r}_i^t(I, a)- r(I, a \mid \phi_i^t)\right)^2 \right] $, the expected target value of $r(I, a \mid \phi_i^t)$ for any sampled information set $I$ is given by:
\begin{align*}
    \mathbb{E}_{z \sim \xi^t}\left[ \hat{r}_i^{t}(I, a) \vert z \in Z_I\right] =  \frac{r_i^t(I, a)}{\pi^{\xi^t}(I)}.
\end{align*}
\label{thm:1}
\end{theorem}

To address these challenges, we adjust the calculation of the sampled counterfactual value as
    $\check{v}_i^{\sigma^t}(I, a \mid z) = \frac{\pi^{\sigma^t}(z[I]a, z)u_i(z)}{\pi^{\xi^t}(z[I], z)}$,
and $\check{v}_i^{\sigma^t}(I\mid z)$ and $\check{r}_i^t(I, a)$ are defined similarly.
As demonstrated in Theorem~\ref{thm:2}, these expectations correspond to the advantages of the information set $I$.
Advancements in deep reinforcement learning have shown that neural networks excel at predicting and generalizing advantages, even in complex environments with large state spaces~\cite{schulman2017proximal}.
By computing the cumulative advantages $\check{R}^t(I, a)=\sum_{k=1}^{t}A^{\sigma^t}(I, a)$ instead of cumulative counterfactual regrets, the process can be interpreted as a type of weighted CFR since $r_i^t(I, a)=\pi_{-i}^{\sigma^t}(I)A_i^{\sigma^t}(I, a) $~\cite{srinivasan2018actor}.
\begin{theorem}
By using outcome sampling to collect data $(I, \check{r}_i^t(I))$ into a buffer $\mathcal{B}_i$ for player $i$ in iteration $t$, and training a neural network $r(I, a \mid \phi_i^t)$ on loss $\mathcal{L}(\phi_i^t)=\mathbb{E}_{(I, \check{r}_i^t(I))\sim \mathcal{B}_i}\left[\sum_{a\in \mathcal{A}(I)}^{}\left(\check{r}_i^t(I, a) - r(I, a \mid \phi_i^t)\right)^2\right] $, the expected target value of $r(I, a \mid \phi_i^t)$ for any sampled information set $I$ is given by:
\begin{align*}
    \mathbb{E}_{z \sim \xi^t}\left[ \check{r}_i^{t}(I, a) \vert z \in Z_I\right] =  \frac{r_i^t(I, a)}{\pi_{-i}^{\xi^t}(I) } =A_i^{\sigma^t}(I, a).
\end{align*}
\label{thm:2}
\end{theorem}
To address the high variance introduced by the importance sampling term 
$\pi^{\sigma^t}_{i}(I) / \pi^{\xi^t}(z) $ in the sampled strategy $\hat{\sigma}_i^t(I, a \mid z)$, which complicates network training, we directly use the strategy $\sigma_i^t(I, a)$ as the sampled strategy $\check{\sigma}_i^t(I, a \mid z)$. 
However, when it is player $i$'s turn to collect data, we save the sampled strategy $\check{\sigma}_{-i}^t(I, a \mid z)$ to the buffer for the opponent player $-i$. The expectation of the cumulative strategy for player $-i$ is given by $\sum_{t=1}^{T}\mathbb{E}_{z \in \xi^t}\left[\check{\sigma}_{-i}^t(I, a\mid z)  \right] = \sum_{t=1}^{T}\pi_i^{\xi^t}(I)\pi_{-i}^{\sigma^t}(I)\sigma_{-i}^t(I, a)$. This can be interpreted as a form of weighted cumulative strategy, where $\pi_{i}^{\xi^t}(I)$ acts as the weight in iteration $t$.

We now describe the training procedures for the cumulative advantage and average strategy network.
Similar to OS-DeepCFR, outcome sampling is used to sample $K$ episodes per iteration, and these experiences are added into replay buffers.
For player $i$, the replay buffer $\mathcal{B}_{V, i}$ stores information sets $I$ and advantage estimates $\check{r}(I, a)$, which are used to train a cumulative advantage network $R(I, a \mid \theta_i^t)$ to approximate cumulative advantages at a given information set. 
Unlike OS-DeepCFR, where the replay buffer retains samples from all iterations, the replay buffer is cleared at the start of each iteration. 
The network $R(I, a \mid \theta_i^t)$ is trained by bootstrapping according to the loss 
\begin{align*}
    \resizebox{1\linewidth}{!}{$
\mathcal{L}(\theta^t_i) = \mathbb{E}_{(I, \check{r} )\sim \mathcal{B}_{V, i}}\left[\sum_{a \in \mathcal{A}(I)}^{}\left(R(I, a \mid \theta^{t-1}_{i} )+\check{r}(I, a)-R(I, a\mid\theta_i^t)\right)^2\right].
$}
\end{align*}
Another buffer $\mathcal{B}_\Pi$ with reservoir sampling stores information sets $I$, iteration numbers $t$, and sampling strategies $\check{\sigma}^t$. It is used to train an average network $\Pi(I, a \mid \psi)$ that approximates the average strategy over all iterations. 
The network $\Pi(I, a\mid \psi)$ is optimized using the following loss:
\begin{align*}
    \mathcal{L}(\psi) = \mathbb{E}_{(I, t, \check{\sigma}^{t})\sim \mathcal{B}_{\Pi}}\left[\sum_{a\in \mathcal{A}(I)}\left(\check{\sigma}^{t}(I, a) - \Pi(I, a \mid \psi)\right)^2\right].
\end{align*}

\subsection{Approximating Advanced CFR Variants}
The estimated cumulative advantages pave the way for approximating the behaviors of DCFR+ and PDCFR+. Both methods update cumulative counterfactual regrets by applying discounting and clipping operations. This results in the loss for the network $R(I, a \mid \theta^t_i)$ :
\begin{align*}
    \resizebox{1\linewidth}{!}{$
    \mathcal{L}(\theta^t_i) = \mathbb{E}_{(I, \check{r} )\sim \mathcal{B}_{V, i}}\left[\sum_{a \in \mathcal{A}(I)}^{}\left(\max(R(I, a \mid \theta^{t-1}_{i} ), 0)
    \frac{(t-1)^\alpha}{(t-1)^\alpha + 1} +\check{r}(I, a)-R(I, a\mid\theta_i^t)\right)^2\right],
    $}
\end{align*}
where the sequence of discounting and clipping is adjusted to facilitate sampling-based approximation of the expectation $\check{r}(I, a)$. 
Since PDCFR+ relies on predicting the instantaneous counterfactual regrets for the next iteration to compute a new strategy, an instantaneous advantage network $r(I, a \mid \phi_i^t)$ is trained for each player $i$. This network estimates the instantaneous advantages in iteration $t$ using  samples from the replay buffer $\mathcal{B}_{V, i}$, with the loss
    $\mathcal{L}(\phi_i^t) = \mathbb{E}_{(I, \check{r} )\sim \mathcal{B}_{V, i}}\left[\sum_{a \in \mathcal{A}(I)}^{}\left(\check{r}(I, a)-r(I, a\mid\phi_i^t)\right)^2\right].$
The instantaneous advantage network is then used to predict the cumulative advantages for the next iteration as $
\max\left(R(I, a \mid \theta^t_i), 0\right) \frac{t^\alpha}{t^\alpha + 1} + r(I, a\mid \phi_i^t)$
For the average strategy, the cumulative strategy can be expressed as $C_i^t(I, a) = C_i^{t-1}(I, a)+t^\gamma\pi_i^{\sigma^t}(I)\sigma^t_i(I, a)$.
So we can train the average strategy network with the loss:
\begin{align*}
    \resizebox{1\linewidth}{!}{$
    \mathcal{L}(\psi) = \mathbb{E}_{(I, t, \sigma^{t})\sim \mathcal{B}_{\Pi}}\left[\left(\frac{t}{T}\right)^\gamma \sum_{a\in \mathcal{A}(I)}\left(\sigma^{t}(I, a) - \Pi(I, a \mid \psi)\right)^2\right],
    $}
\end{align*}
where $T$ is the total number of iterations.

\subsection{Variance Reduction Based on Baseline Functions}

To mitigate the high variance caused by episode sampling, prior works such as DREAM~\cite{steinberger2020dream} and ESCHER~\cite{ESCHER} incorporate value functions at history nodes. 
DREAM uses value functions as baseline functions for each action, and constructing an unbiased estimator of counterfactual regrets, while ESCHER directly computes counterfactual regrets using value functions.
Since ESCHER requires accurate value estimation and thus incurs significant training time, we adopt the variance reduction approach of DREAM.

For each episode $z$ in iteration $t$, we extract and store a set of experience tuples $(t, h, \hat{a}, \hat{u}, h', I', i)$ in the history value buffer $\mathcal{B}_Q$. Each tuple represents player $i$ taking action $\hat{a}$ at history node $h$, transitioning to node $h'$ associated with information set $I'$, and player 1 receiving utility $\hat{u}$ (where $\hat{u} = u_1(h')$ if $h'$ is a terminal history, and $\hat{u}=0$ otherwise).
The network $Q(h, a \mid w^t)$ estimates the value of each action for player 1 at every history node under the strategy $\sigma^{t+1}$. 
It is trained with the loss
\begin{align*}
    \resizebox{1\linewidth}{!}{$
\mathcal{L}(\omega^t) = \mathbb{E}_{(t, h, \hat{a}, \hat{u}, h', I', i)\sim \mathcal{B}_{Q}}\left[\left(\hat{u} + \sum_{a'\in \mathcal{A}(h')}^{}\sigma_i^{t+1}(I', a')Q(h', a' \mid \omega') - Q(h, \hat{a} \mid \omega^{t}) \right)^2\right].
    $}
\end{align*}
The network is trained in an off-policy manner, eliminating the need to sample new episodes under $\sigma^{t+1}$, thereby improving learning efficiency.
Since we assume the game is two-player zero-sum, we have $Q_1(h, a \mid w^t) = Q(h, a \mid w)$ and $Q_2(h, a \mid w^t) = -Q(h, a \mid w)$.
It is used to compute baseline-adjusted sampled value 
\begin{align*}
    \resizebox{1\linewidth}{!}{$
    \bar{v}_i^{\sigma^t}(I, a \mid z)  =  \left\{
        \begin{array}{ll}
         Q_i(h, a \mid w^{t-1} ) + \frac{\bar{v}_i^{\sigma^t}(I' \mid z) - Q_i(h, a \mid w^{t-1})}{\xi^t(I, a)} & \text{if}\;a=\hat{a}\\
         Q_i(h, a \mid w^{t-1} ) & \text{otherwise,}
        \end{array}
        \right.
    $}
\end{align*}

\begin{align*}
    \resizebox{1\linewidth}{!}{$
    \bar{v}_i^{\sigma^t}(I \mid z) = \left\{
        \begin{array}{ll}
          u_i(z) & \text{if}\;h=z\\
            \sum_{a\in \mathcal{A}(h)}^{}\sigma_i^{t}(I, a) \bar{v}_i^{\sigma^t}(I, a \mid z) & \text{otherwise.}
        \end{array}
        \right.
    $}
\end{align*}
We then replace the sampled advantages $\check{r}_i^t(I, a \mid z)$ with baseline-adjusted sampled advantages $\bar{r}_i^t(I, a \mid z) = \bar{v}_i^{\sigma^t}(I, a \mid z) - \bar{v}_i^{\sigma^t}(I \mid z)$, which serve as unbiased estimators~\cite{schmid2019variance}.
Algorithm~\ref{algo1} outlines the complete training procedure of our algorithms.

\begin{figure*}[t]
    \centering
    \includegraphics[width=1\linewidth]{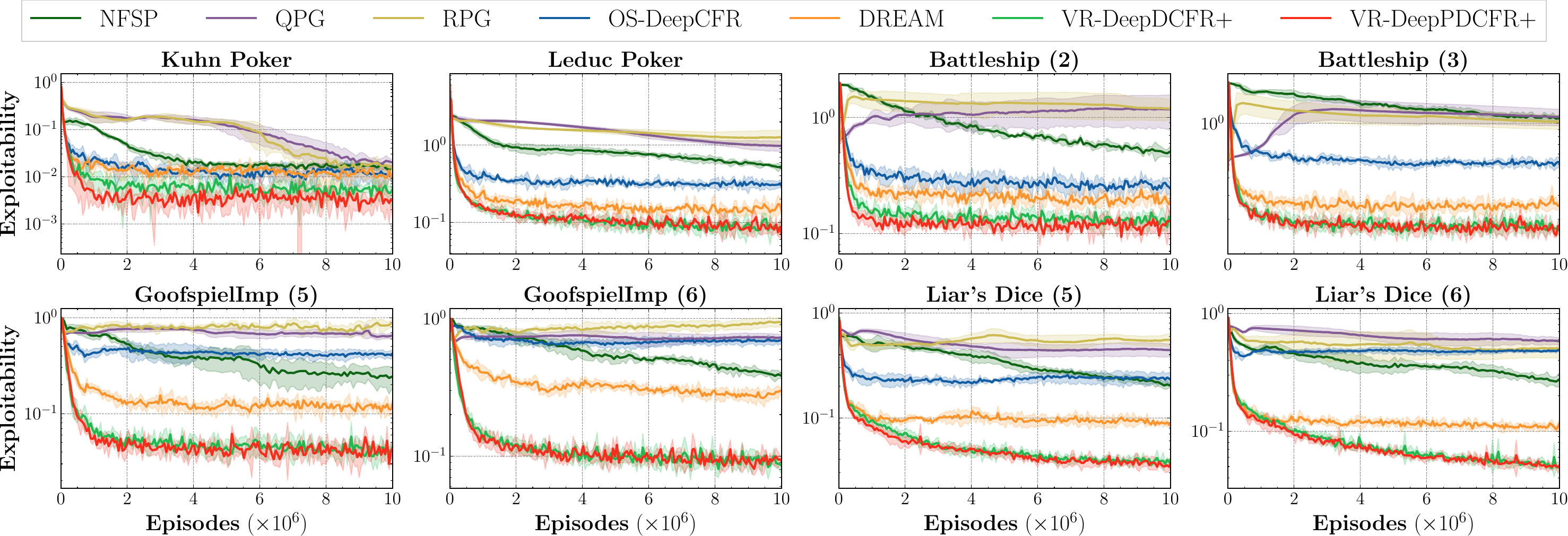}
    \caption{
        Convergence results of seven model-free neural algorithms on eight testing games. 
    }
    \vspace{-0.3cm}
    \label{fig:compare}
\end{figure*}

\section{Experiments}

In this section, we evaluate the performance of VR-DeepDCFR+ and VR-DeepPDCFR+ through extensive experiments. 
We first demonstrate their empirical convergence toward the NE across eight widely used IIGs in the research community. 
We provide detailed descriptions of the games in Appendix~\ref{app:games}.
Exploitability is used as the performance metric to showcase the convergence speeds. We then conduct experiments on a large poker game. 
Given the large size of the game, we assess performance by playing against five agents with different styles, using average rewards as the performance metric.
All testing games are sourced from OpenSpiel~\cite{lanctot2019openspiel}.

We compare our methods against five model-free neural methods: NFSP, q-based policy gradient (QPG) / regret policy gradient (RPG)~\cite{srinivasan2018actor}, OS-DeepCFR and DREAM. All methods use similar network architectures with three hidden layers with 64 neurons each. 
We normalize the utilities received by all algorithms in each game to the range $[-1, 1]$.
The implementation of NFSP, QPG, and RPG are sourced from OpenSpiel, while OS-DeepCFR is adapted from ES-DeepCFR in OpenSpiel.
Hyperparameters for NFSP and OS-DeepCFR follow the OpenSpiel reproduction report~\cite{walton2021multi}, while those for QPG/RPG are from ~\cite{farina2021model}. For DREAM, we adopt the     base hyperparameters of OS-DeepCFR and use a circular buffer of size 1,000,000 for the history value network. VR-DeepDCFR+ and VR-DeepPDCFR+ share same common hyperparameters with DREAM, with specific settings of $\alpha\small{=}2, \gamma\small{=}2$ for DeepDCFR+ and $\alpha\small{=}2.3, \gamma\small{=}2$ for VR-DeepPDCFR+. All algorithms use the same hyperparameters across all games. Detailed configurations are provided in Appendix~\ref{app:hypers}.

\subsection{Convergence to Equilibrium}

We run each algorithm four times with different random seeds, and the results are shown in Figure~\ref{fig:compare}. 
In all plots, the x-axis is the number of episodes sampled by each algorithm, and the y-axis is exploitability shown on a log scale. The shaded area represents $95\%$ confidence intervals over four random seeds.
The two policy gradient algorithms QPG and RPG converge to an exploitability of 0.01 in \textit{Kuhn Poker}, but perform poorly in more complex games, consistent with findings in the work~\citep{farina2021model}.
Neural CFR variants generally converge faster than NFSP, mainly due to CFR's theoretically superior convergence rate. 
Compared to OS-DeepCFR, the algorithms proposed in this work converge faster in most games, demonstrating the effectiveness of approximating cumulative advantages.
Cumulative advantages exhibit less variance than cumulative counterfactual regrets, leading to more stable neural network training and faster convergence of the average strategy.
Moreover, VR-DeepDCFR+ and VR-DeepPDCFR+ inherit the convergence advantages of DCFR+ and PDCFR+ over vanilla CFR and LinearCFR.
Experimental results show that VR-DeepDCFR+ and VR-DeepPDCFR+ outperform OS-DeepCFR and DREAM in convergence speed across most games, highlighting the benefit of integrating neural networks with advanced CFR variants.
The running time of VR-DeepDCFR+ is roughly the same as that of DREAM, since the main difference lies in their loss formulations, which incur negligible overhead. 
Moreover, by using bootstrapping instead of retraining from scratch, it actually requires fewer total training steps.

\subsection{Head-to-Head Evaluation}

We evaluate the algorithms on the large poker game flop hold'em poker (\textbf{\textit{FHP}}) by playing 20,000 matches against five rule-based agents with different styles.
Each agent estimates hand strength at decision points and follows predefined rules reflecting different degrees of aggressiveness, tightness, or bluffing.
These agents simulate diverse exploit scenarios, allowing a multidimensional assessment of strategy robustness~\citep{li2018opponent}. Please refer to Appendix~\ref{app:rule_agents} for details.
We use the average reward over these matches as the performance metric.

Given the poor performance of NFSP, QPG, and RPG in typical IIGs, we focus on comparing four neural CFR variants.
We increase the number of sampled episodes to $10^8$, buffer size to $10^7$, neurons per layer to 128, while keeping other hyperparameters unchanged.
The results are shown in Figure~\ref{fig:deepcfr_head_to_head}. 
Final average rewards are $-7.8\pm 1.4$ chips for OS-DeepCFR, $-2.0\pm 3.1$ for DREAM, $11.6\pm 1.2$ for VR-DeepDCFR+, and $11.3\pm 0.9$ for VR-DeepPDCFR+.
Among the four methods, VR-DeepDCFR+ and VR-DeepPDCFR+ consistently outperform various rule-based agents with different styles.
Notably, in professional Texas Hold’em matches, an average reward of five chips per hand is considered a significant skill gap~\citep{moravvcik2017deepstack}.
Therefore, compared to other neural CFR variants, the proposed methods demonstrate higher learning efficiency and superior performance in the large poker game.

\begin{figure}[h]
    \centering
    \includegraphics[width=1\linewidth]{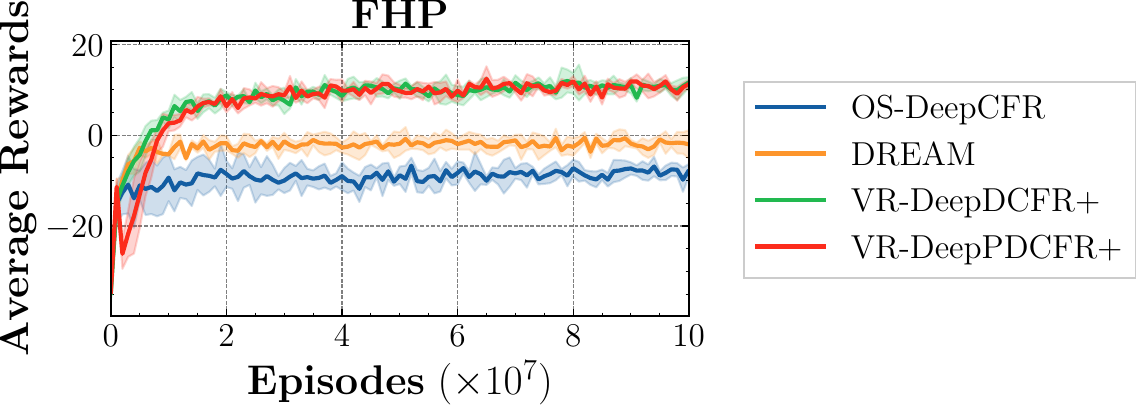}
    \vspace{-0.5cm}
    \caption{
        Head-to-head evaluation results of four neural CFR variants on \textit{FHP}.
    }
    \vspace{-0.4cm}
    \label{fig:deepcfr_head_to_head}
\end{figure}

\subsection{Ablation Studies}
The proposed algorithms consist of three key components: bootstrapped cumulative advantages estimation, approximating advanced CFR variants, and baseline-based variance reduction.
To evaluate the impact of each component, we use VR-DeepPDCFR+ as the base method and test performance on four IIGs after removing each module individually.
Experimental results show that all three components contribute to improved learning efficiency and overall performance.
Please refer to Appendix~\ref{app:ablation} for details.

\section{Conclusions and Future Research}

This work proposes two novel model-free neural CFR variants, VR-DeepDCFR+ and VR-DeepPDCFR+, for learning an NE in two-player zero-sum IIGs. 
In each iteration, the algorithms collect variance-reduced sampled advantages using a history value network, bootstrap cumulative advantages, and apply discounting and clipping to simulate the behaviors of advanced tabular CFR variants DCFR+ and PDCFR+. 
Experimental results demonstrate that our methods achieve faster convergence across eight widely used IIGs and obtain higher average rewards against various rule-based agents in a large poker game compared to other model-free neural algorithms.
Several promising directions remain for future work. One potential avenue is to improve the prediction of instantaneous advantages in VR-DeepPDCFR+, possibly by leveraging recurrent neural networks to capture temporal dependencies more effectively~\cite{sychrovsky2024learning}.

\section{Acknowledgments}
This work is supported in part by the National Key R\&D Program of China (No. 2025ZD0122000), the Natural Science Foundation of China (Nos. 62222606 and 61902402), the Key Research and Development Program of Jiangsu Province (No. BE2023016), and the CCF-Baidu Open Fund.

\bibliography{aaai2026}

@article{nash1950equilibrium,
  title   = {Equilibrium points in n-person games},
  author  = {Nash, Jr John F},
  year    = {1950},
  journal = {Proceedings of the National Academy of Sciences of the United States of America},
  volume  = {36},
  number  = {1},
  pages   = {48--49}
}

@article{moravvcik2017deepstack,
  title   = {{DeepStack}: Expert-level artificial intelligence in heads-up no-limit poker},
  author  = {Morav{\v{c}}{\'\i}k, Matej and Schmid, Martin and Burch, Neil and Lis{\`y}, Viliam and Morrill, Dustin and Bard, Nolan and Davis, Trevor and Waugh, Kevin and Johanson, Michael and Bowling, Michael},
  year    = {2017},
  journal = {Science},
  volume  = {356},
  number  = {6337},
  pages   = {508--513}
}

@article{brown2018superhuman,
  title   = {Superhuman {AI} for heads-up no-limit poker: Libratus beats top professionals},
  author  = {Brown, Noam and Sandholm, Tuomas},
  year    = {2018},
  journal = {Science},
  volume  = {359},
  number  = {6374},
  pages   = {418--424}
}

@article{brown2019superhuman,
  title   = {Superhuman {AI} for multiplayer poker},
  author  = {Brown, Noam and Sandholm, Tuomas},
  year    = {2019},
  journal = {Science},
  volume  = {365},
  number  = {6456},
  pages   = {885--890}
}

@article{bowling_heads-up_2015,
  title   = {Heads-up limit hold'em poker is solved},
  author  = {Bowling, Michael and Burch, Neil and Johanson, Michael and Tammelin, Oskari},
  year    = {2015},
  journal = {Science},
  volume  = {347},
  number  = {6218},
  pages   = {145--149}
}

@article{hart2000simple,
  title   = {A simple adaptive procedure leading to correlated equilibrium},
  author  = {Hart, Sergiu and Mas-Colell, Andreu},
  year    = {2000},
  journal = {Econometrica},
  volume  = {68},
  number  = {5},
  pages   = {1127--1150}
}

@article{kuhn1950simplified,
  title   = {A simplified two-person poker},
  author  = {Kuhn, Harold W},
  year    = {1950},
  journal = {Contributions to the Theory of Games},
  volume  = {1},
  pages   = {97--103}
}

@article{ross1971goofspiel,
  title   = {Goofspiel—the game of pure strategy},
  author  = {Ross, Sheldon M},
  year    = {1971},
  journal = {Journal of Applied Probability},
  volume  = {8},
  number  = {3},
  pages   = {621--625}
}

@misc{tammelin_solving_2014,
      title={Solving Large Imperfect Information Games Using CFR+}, 
      author={Oskari Tammelin},
      year={2014},
      eprint={1407.5042},
      archivePrefix={arXiv},
      primaryClass={cs.GT},
}

@misc{schulman2017proximal,
      title={Proximal Policy Optimization Algorithms}, 
      author={John Schulman and Filip Wolski and Prafulla Dhariwal and Alec Radford and Oleg Klimov},
      year={2017},
      eprint={1707.06347},
      archivePrefix={arXiv},
      primaryClass={cs.LG},
}

@inproceedings{zinkevich_regret_2007,
  title     = {Regret minimization in games with incomplete information},
  author    = {Zinkevich, Martin and Johanson, Michael and Bowling, Michael and Piccione, Carmelo},
  year      = {2007},
  booktitle = {Advances in Neural Information Processing Systems},
  pages     = {1729--1736}
}

@inproceedings{farina2019correlation,
  title     = {Correlation in extensive-form games: Saddle-point formulation and benchmarks},
  author    = {Farina, Gabriele and Ling, Chun Kai and Fang, Fei and Sandholm, Tuomas},
  year      = {2019},
  booktitle = {Advances in Neural Information Processing Systems},
  pages     = {9229--9239}
}

@inproceedings{brown_solving_2019,
  title     = {Solving imperfect-information games via discounted regret minimization},
  author    = {Brown, Noam and Sandholm, Tuomas},
  year      = {2019},
  booktitle = {AAAI Conference on Artificial Intelligence},
  pages     = {1829--1836}
}

@inproceedings{farina2021faster,
  title     = {Faster game solving via predictive Blackwell approachability: Connecting regret matching and mirror descent},
  author    = {Farina, Gabriele and Kroer, Christian and Sandholm, Tuomas},
  booktitle = {AAAI Conference on Artificial Intelligence},
  pages     = {5363--5371},
  year      = {2021}
}

@inproceedings{brown2019deep,
  title     = {Deep counterfactual regret minimization},
  author    = {Brown, Noam and Lerer, Adam and Gross, Sam and Sandholm, Tuomas},
  booktitle = {International Conference on Machine Learning},
  pages     = {793--802},
  year      = {2019}
}

@inproceedings{southey2005bayes,
  title     = {Bayes' bluff: Opponent modelling in poker},
  author    = {Southey, Finnegan and Bowling, Michael and Larson, Bryce and Piccione, Carmelo and Burch, Neil and Billings, Darse and Rayner, Chris},
  year      = {2005},
  booktitle = {Conference on Uncertainty in Artificial Intelligence},
  pages     = {550--558}
}

@inproceedings{lisy2015online,
  title     = {Online Monte Carlo counterfactual regret minimization for search in imperfect information games},
  author    = {Lis{\`y}, Viliam and Lanctot, Marc and Bowling, Michael H},
  year      = {2015},
  booktitle = {International Conference on Autonomous Agents and Multiagent Systems},
  pages     = {27--36}
}

@inproceedings{lanctot2009monte,
  title     = {Monte Carlo sampling for regret minimization in extensive games},
  author    = {Lanctot, Marc and Waugh, Kevin and Zinkevich, Martin and Bowling, Michael},
  year      = {2009},
  booktitle = {Advances in Neural Information Processing Systems},
  pages     = {1078--1086}
}

@inproceedings{sandholm2015steering,
  title     = {Steering evolution strategically: Computational game theory and opponent exploitation for treatment planning, drug design, and synthetic biology},
  author    = {Sandholm, Tuomas},
  booktitle = {AAAI Conference on Artificial Intelligence},
  year      = {2015},
  pages     = {4057-4061}
}

@inproceedings{lisy2016counterfactual,
  title     = {Counterfactual regret minimization in sequential security games},
  author    = {Lisy, Viliam and Davis, Trevor and Bowling, Michael},
  booktitle = {AAAI Conference on Artificial Intelligence},
  year      = {2016},
  pages     = {544-550}
}

@inproceedings{gratch2016misrepresentation,
  title     = {The misrepresentation game: How to win at negotiation while seeming like a nice guy},
  author    = {Gratch, Jonathan and Nazari, Zahra and Johnson, Emmanuel},
  booktitle = {International Conference on Autonomous Agents and Multiagent Systems},
  pages     = {728--737},
  year      = {2016}
}

@book{osborne1994course,
  title     = {A course in game theory},
  author    = {Osborne, Martin J and Rubinstein, Ariel},
  year      = {1994},
  publisher = {MIT press}
}

@inproceedings{PDCFRPlus,
  title     = {Minimizing Weighted Counterfactual Regret with Optimistic Online Mirror Descent},
  author    = {Hang, Xu and Kai, Li and Bingyun, Liu and Haobo, Fu and Qiang, Fu and Junliang, Xing and Jian Cheng},
  booktitle = {International Joint Conference on Artificial Intelligence},
  year      = {2024},
  pages     = {5272--5280}
}

@misc{steinberger2020dream,
      title={DREAM: Deep Regret minimization with Advantage baselines and Model-free learning}, 
      author={Eric Steinberger and Adam Lerer and Noam Brown},
      year={2020},
      eprint={2006.10410},
      archivePrefix={arXiv},
      primaryClass={cs.LG},
}

@inproceedings{ESCHER,
  title     = {ESCHER: Eschewing importance sampling in games by computing a history value function to estimate regret},
  author    = {McAleer, Stephen and Farina, Gabriele and Lanctot, Marc and Sandholm, Tuomas},
  year      = {2023},
  booktitle = {International Conference on Learning Representations},
  pages     = {1-22}
}

@inproceedings{schmid2019variance,
  title     = {Variance reduction in Monte Carlo counterfactual regret minimization (VR-MCCFR) for extensive form games using baselines},
  author    = {Schmid, Martin and Burch, Neil and Lanctot, Marc and Moravcik, Matej and Kadlec, Rudolf and Bowling, Michael},
  booktitle = {AAAI Conference on Artificial Intelligence},
  pages     = {2157--2164},
  year      = {2019}
}

@inproceedings{AutoCFR,
  title     = {AutoCFR: Learning to Design Counterfactual Regret Minimization Algorithms},
  author    = {Hang, Xu and Kai, Li and Haobo, Fu and Qiang, Fu and Junliang, Xing},
  booktitle = {AAAI Conference on Artificial Intelligence},
  year      = {2022},
  pages     = {5244--5251}
}

@inproceedings{fu2021actor,
  title     = {Actor-critic policy optimization in a large-scale imperfect-information game},
  author    = {Fu, Haobo and Liu, Weiming and Wu, Shuang and Wang, Yijia and Yang, Tao and Li, Kai and Xing, Junliang and Li, Bin and Ma, Bo and Fu, Qiang and others},
  booktitle = {International Conference on Learning Representations},
  year      = {2022},
  pages     = {1--28}
}

@inproceedings{van2016learning,
  title     = {Learning values across many orders of magnitude},
  author    = {Van Hasselt, Hado P and Guez, Arthur and Hessel, Matteo and Mnih, Volodymyr and Silver, David},
  booktitle = {Advances in Neural Information Processing Systems},
  year      = {2016},
  pages     = {4294-4302}
}

@inproceedings{srinivasan2018actor,
  title     = {Actor-critic policy optimization in partially observable multiagent environments},
  author    = {Srinivasan, Sriram and Lanctot, Marc and Zambaldi, Vinicius and P{\'e}rolat, Julien and Tuyls, Karl and Munos, R{\'e}mi and Bowling, Michael},
  booktitle = {Advances in Neural Information Processing Systems},
  year      = {2018},
  pages     = {3426-3439}
}

@misc{lanctot2019openspiel,
      title={OpenSpiel: A Framework for Reinforcement Learning in Games}, 
      author={Marc Lanctot and Edward Lockhart and Jean-Baptiste Lespiau and Vinicius Zambaldi and Satyaki Upadhyay and Julien Pérolat and Sriram Srinivasan and Finbarr Timbers and Karl Tuyls and Shayegan Omidshafiei and Daniel Hennes and Dustin Morrill and Paul Muller and Timo Ewalds and Ryan Faulkner and János Kramár and Bart De Vylder and Brennan Saeta and James Bradbury and David Ding and Sebastian Borgeaud and Matthew Lai and Julian Schrittwieser and Thomas Anthony and Edward Hughes and Ivo Danihelka and Jonah Ryan-Davis},
      year={2020},
      eprint={1908.09453},
      archivePrefix={arXiv},
      primaryClass={cs.LG},
}

@misc{heinrich2016deep,
      title={Deep Reinforcement Learning from Self-Play in Imperfect-Information Games}, 
      author={Johannes Heinrich and David Silver},
      year={2016},
      eprint={1603.01121},
      archivePrefix={arXiv},
      primaryClass={cs.LG},
}

@misc{walton2021multi,
      title={Multi-agent Reinforcement Learning in OpenSpiel: A Reproduction Report}, 
      author={Michael Walton and Viliam Lisy},
      year={2021},
      eprint={2103.00187},
      archivePrefix={arXiv},
      primaryClass={cs.AI},
      url={https://arxiv.org/abs/2103.00187}, 
}

@inproceedings{farina2021model,
  title     = {Model-free online learning in unknown sequential decision making problems and games},
  author    = {Farina, Gabriele and Sandholm, Tuomas},
  booktitle = {AAAI Conference on Artificial Intelligence},
  pages     = {5381--5390},
  year      = {2021}
}

@inproceedings{sychrovsky2024learning,
  title     = {Learning not to regret},
  author    = {Sychrovsk{\`y}, David and {\v{S}}ustr, Michal and Davoodi, Elnaz and Bowling, Michael and Lanctot, Marc and Schmid, Martin},
  booktitle = {AAAI Conference on Artificial Intelligence},
  pages     = {15202--15210},
  year      = {2024}
}

@inproceedings{lanctot2017unified,
  title     = {A unified game-theoretic approach to multiagent reinforcement learning},
  author    = {Lanctot, Marc and Zambaldi, Vinicius and Gruslys, Audrunas and Lazaridou, Angeliki and Tuyls, Karl and P{\'e}rolat, Julien and Silver, David and Graepel, Thore},
  booktitle = {Advances in Neural Information Processing Systems},
  year      = {2017},
  pages     = {4193--4206}
}

@inproceedings{lidouble,
  title     = {Double Neural Counterfactual Regret Minimization},
  author    = {Li, Hui and Hu, Kailiang and Zhang, Shaohua and Qi, Yuan and Song, Le},
  booktitle = {International Conference on Learning Representations},
  year      = {2020},
  pages     = {1--20}
}

@misc{gruslys2020advantage,
      title={The Advantage Regret-Matching Actor-Critic}, 
      author={Audrūnas Gruslys and Marc Lanctot and Rémi Munos and Finbarr Timbers and Martin Schmid and Julien Perolat and Dustin Morrill and Vinicius Zambaldi and Jean-Baptiste Lespiau and John Schultz and Mohammad Gheshlaghi Azar and Michael Bowling and Karl Tuyls},
      year={2020},
      eprint={2008.12234},
      archivePrefix={arXiv},
      primaryClass={cs.AI}, 
}

@article{liu2022model,
  title   = {Model-free neural counterfactual regret minimization with bootstrap learning},
  author  = {Liu, Weiming and Li, Bin and Togelius, Julian},
  journal = {IEEE Transactions on Games},
  volume  = {15},
  number  = {3},
  pages   = {315--325},
  year    = {2022}
}

@inproceedings{meng2023efficient,
  title={An efficient deep reinforcement learning algorithm for solving imperfect information extensive-form games},
  author={Meng, Linjian and Ge, Zhenxing and Tian, Pinzhuo and An, Bo and Gao, Yang},
  booktitle={AAAI Conference on Artificial Intelligence},
  pages={5823--5831},
  year={2023}
}

@inproceedings{hennes2020neural,
  title={Neural replicator dynamics: Multiagent learning via hedging policy gradients},
  author={Hennes, Daniel and Morrill, Dustin and Omidshafiei, Shayegan and Munos, R{\'e}mi and Perolat, Julien and Lanctot, Marc and Gruslys, Audrunas and Lespiau, Jean-Baptiste and Parmas, Paavo and Du{\'e}{\~n}ez-Guzm{\'a}n, Edgar and others},
  booktitle={International Conference on Autonomous Agents and Multiagent Systems},
  pages={492--501},
  year={2020}
}

@inproceedings{lockhart2019computing,
  title={Computing approximate equilibria in sequential adversarial games by exploitability descent},
  author={Lockhart, Edward and Lanctot, Marc and P{\'e}rolat, Julien and Lespiau, Jean-Baptiste and Morrill, Dustin and Timbers, Finbarr and Tuyls, Karl},
  booktitle={International Joint Conference on Artificial Intelligence},
  pages={464--470},
  year={2019}
}

@misc{johanson2013measuring,
      title={Measuring the Size of Large No-Limit Poker Games}, 
      author={Michael Johanson},
      year={2013},
      eprint={1302.7008},
      archivePrefix={arXiv},
      primaryClass={cs.GT},
}

@inproceedings{li2018opponent,
  title={Opponent modeling and exploitation in poker using evolved recurrent neural networks},
  author={Li, Xun and Miikkulainen, Risto},
  booktitle={Genetic and Evolutionary Computation Conference},
  pages={189--196},
  year={2018}
}

\newpage
\onecolumn
\appendix

\makeatletter
\def\section{\@startsection {section}{1}{\z@}{-2.0ex plus
-0.5ex minus -.2ex}{3pt plus 2pt minus 1pt}{\Large\bf\raggedright}}
\makeatother

\setcounter{secnumdepth}{2}
\section{Proof of Theorems}
\label{app:proofs}
\subsection{Proof of Theorem 1}
\begin{theorem}    
    By using outcome sampling to collect data $(I, \hat{r}_i^t(I))$ into a buffer $\mathcal{B}_i$ for player $i$ in iteration $t$, and training a neural network $r(I, a \mid \phi^t_i)$ on loss $\mathcal{L}(\phi_i^t)=\mathbb{E}_{(I, \hat{r}_i^t(I))\sim \mathcal{B}_i}\left[\sum_{a\in \mathcal{A}(I)}^{}\left(\hat{r}_i^t(I, a)^2 - r(I, a \mid \phi_i^t)\right)\right] $, the expected target value of $r(I, a \mid \phi_i^t)$ for any sampled information set $I$ is given by:
\begin{align*}
    \mathbb{E}_{z \sim \xi^t}\left[ \hat{r}_i^{t}(I, a) \vert z \in Z_I\right] =  \frac{r_i^t(I, a)}{\pi^{\xi^t}(I)}
\end{align*}
\end{theorem}
\begin{proof}

We have
\begin{align*}
    \mathbb{E}_{z \sim \xi^t}\left[ \hat{v}_i^{\sigma^t}(I, a) \vert z \in Z_I\right] =& \frac{\mathbb{E}_{z \sim \xi^t}\left[\hat{v}_i^{\sigma^t}(I)\right]}{P_{z \in \xi^t}(z \in Z_I)} \\
    =& \frac{\sum_{z \in Z_I}^{}\pi^{\xi^t}(z) * \frac{\pi_{-i}^{\sigma^t}(z[I])\pi^{\sigma^t}(z[I]a, z)u_i(z)}{\pi^{\xi^t}(z)}}{\sum_{z \in Z_I}^{}\pi^{\xi^t}(z)} \\
    =& \frac{\sum_{z \in Z_I}^{}\pi_{-i}^{\sigma^t}(z[I])\pi^{\sigma^t}(z[I]a, z)u_i(z)}{\sum_{z \in Z_I}^{}\pi^{\xi^t}(z)} \\
    =& \frac{v_i^{\sigma^t}(I, a)}{\pi^{\xi^t}(I)},
\end{align*}
and
\begin{align*}
    \mathbb{E}_{z \sim \xi^t}\left[ \hat{v}_i^{\sigma^t}(I) \vert z \in Z_I\right]  &= \mathbb{E}_{z \sim \xi^t}\left[ \sum_{a \in \mathcal{A}(I)}^{}\sigma^t_i(I, a)\hat{v}_i^{\sigma^t}(I, a) \biggm| z \in Z_I\right] \\
    &= \sum_{a \in \mathcal{A}(I)}^{}\sigma_i^t(I, a)\mathbb{E}_{z \sim \xi^t}\left[ \hat{v}_i^{\sigma^t}(I, a) \vert z \in Z_I\right] \\
    &= \frac{\sum_{a \in \mathcal{A}(I)}^{}\sigma_i^t(I, a)v_i^{\sigma^t}(I, a)}{\pi^{\xi^t}(I)} \\
    &= \frac{v_i^{\sigma^t}(I)}{\pi^{\xi^t}(I)}.
\end{align*}
Hence,
\begin{align*}
    \mathbb{E}_{z \sim \xi^t}\left[ \hat{r}_i^{t}(I, a) \vert z \in Z_I\right] &= \mathbb{E}_{z \sim \xi^t}\left[ \hat{v}_i^{\sigma^t}(I, a) - \hat{v}_i^{\sigma^t}(I)\vert z \in Z_I\right] \\
    &= \mathbb{E}_{z \sim \xi^t}\left[ \hat{v}_i^{\sigma^t}(I, a) \vert z \in Z_I\right] - \mathbb{E}_{z \sim \xi^t}\left[ \hat{v}_i^{\sigma^t}(I) \vert z \in Z_I\right] \\
    &= \frac{v_i^{\sigma^t}(I, a)}{\pi^{\xi^t}(I)} - \frac{v_i^{\sigma^t}(I)}{\pi^{\xi^t}(I)} \\
    &= \frac{r_i^t(I, a)}{\pi^{\xi^t}(I)}.
\end{align*}
\end{proof}

\subsection{Proof of Theorem 2}
\begin{theorem}
    By using outcome sampling to collect data $(I, \check{r}_i^t(I))$ into a buffer $\mathcal{B}_i$ for player $i$ in iteration $t$, and training a neural network $r(I, a \mid \phi_i^t)$ on loss $\mathcal{L}(\phi_i^t)=\mathbb{E}_{(I, \check{r}_i^t(I))\sim \mathcal{B}_i}\left[\sum_{a\in \mathcal{A}(I)}^{}\left(\check{r}_i^t(I, a) - r(I, a \mid \phi_i^t)\right)^2\right] $, the expected target value of $r(I, a \mid \phi_i^t)$ for any sampled information set $I$ is given by:
\begin{align*}
    \mathbb{E}_{z \sim \xi^t}\left[ \check{r}_i^{t}(I, a) \vert z \in Z_I\right] =  \frac{r_i^t(I, a)}{\pi_{-i}^{\xi^t}(I) } =A_i^{\sigma^t}(I, a)
\end{align*}
\end{theorem}
\begin{proof}
    We have
    \begin{align*}
        \mathbb{E}_{z \sim \xi^t}\left[ \check{v}_i^{\sigma^t}(I, a) \vert z \in Z_I\right] =& \frac{\mathbb{E}_{z \sim \xi^t}\left[\check{v}_i^{\sigma^t}(I)\right]}{P_{z \in \xi^t}(z \in Z_I)} \\
        =& \frac{\sum_{z \in Z_I}^{}\pi^{\xi^t}(z) * \frac{\pi^{\sigma^t}(z[I]a, z)u_i(z)}{\pi^{\xi^t}(z[I], z)}}{\sum_{z \in Z_I}^{}\pi^{\xi^t}(z)} \\
        =& \frac{\sum_{z \in Z_I}^{}\pi^{\xi^t}(z[I])\pi^{\sigma^t}(z[I]a, z)u_i(z)}{\pi^{\xi^t}(I)} \\
        =& \frac{\sum_{z \in Z_I}^{}\pi_i^{\xi^t}(z[I])\pi_{-i}^{\sigma^t}(z[I])\pi^{\sigma^t}(z[I]a, z)u_i(z)}{\pi_i^{\xi^t}(I)\pi_{-i}^{\sigma^t}(I)} \\
        =& \frac{\pi_i^{\xi^t}(I)\sum_{z \in Z_I}^{}\pi_{-i}^{\sigma^t}(z[I])\pi^{\sigma^t}(z[I]a, z)u_i(z)}{\pi_i^{\xi^t}(I)\pi_{-i}^{\sigma^t}(I)} \\
        =& \frac{v_i^t(I, a)}{\pi_{-i}^{\sigma^t}(I)},
    \end{align*}
    where we use $\xi^t_{-i} = \sigma_{-i}^t$ in the fourth equity and $\forall h \in I, \pi_i^{\xi^t}(I) = \pi_i^{\xi^t}(h)$ in the fifth equity. Besides, 
    \begin{align*}
        \mathbb{E}_{z \sim \xi^t}\left[ \check{v}_i^{\sigma^t}(I) \vert z \in Z_I\right]  &= \mathbb{E}_{z \sim \xi^t}\left[ \sum_{a \in \mathcal{A}(I)}^{}\sigma^t_i(I, a)\check{v}_i^{\sigma^t}(I, a) \biggm| z \in Z_I\right] \\
        &= \sum_{a \in \mathcal{A}(I)}^{}\sigma_i^t(I, a)\mathbb{E}_{z \sim \xi^t}\left[ \check{v}_i^{\sigma^t}(I, a) \vert z \in Z_I\right] \\
        &= \frac{\sum_{a \in \mathcal{A}(I)}^{}\sigma_i^t(I, a)v_i^{\sigma^t}(I, a)}{\pi^{\xi^t}(I)} \\
        &= \frac{v_i^{\sigma^t}(I)}{\pi_{-i}^{\sigma^t}(I)}.
    \end{align*}
    Hence,
    \begin{align*}
        \mathbb{E}_{z \sim \xi^t}\left[ \check{r}_i^{t}(I, a) \vert z \in Z_I\right] &= \mathbb{E}_{z \sim \xi^t}\left[ \check{v}_i^{\sigma^t}(I, a) - \check{v}_i^{\sigma^t}(I)\vert z \in Z_I\right] \\
        &= \mathbb{E}_{z \sim \xi^t}\left[ \check{v}_i^{\sigma^t}(I, a) \vert z \in Z_I\right] - \mathbb{E}_{z \sim \xi^t}\left[ \check{v}_i^{\sigma^t}(I) \vert z \in Z_I\right] \\
        &= \frac{v_i^{\sigma^t}(I, a)}{\pi_{-i}^{\sigma^t}(I)} - \frac{v_i^{\sigma^t}(I)}{\pi_{-i}^{\sigma^t}(I)} \\
        &= \frac{r_i^t(I, a)}{\pi_{-i}^{\sigma^t}(I)} \\
        &= A_i^{\sigma^t}(I, a).
    \end{align*}    
\end{proof}

\section{Description of the Games}
\label{app:games}
\textbf{\textit{Kuhn Poker}} is a simplified form of poker proposed by Harold W. Kuhn~\cite{kuhn1950simplified}. 
The game employs a deck of three cards, represented by J, Q, K. 
At the beginning of the game, each player receives a private card drawn from a shuffled deck and places one chip into the pot.
The game involves four kinds of actions: 1) fold, giving up the current game, and the other player gets all the pot, 2) call, increasing his/her bet until both players have the same chips, 3) bet, putting more chips to the pot, and 4) check, declining to wager any chips when not facing a bet. 
In \textit{Kuhn Poker}, each player has an opportunity to bet one chip. 
If neither player folds, both players reveal their cards, and the player holding the higher card takes all the chips in the pot.
The utility for each player is defined as the difference between the number of chips after playing and the number of chips before playing.

\textbf{\textit{Leduc Poker}} is a larger poker game first introduced in~\cite{southey2005bayes}. 
The game uses six cards that include two suites, each comprising three ranks (Js, Qs, Ks, Jh, Qh, Kh).
Similar to \textit{Kuhn Poker}, each player initially bets one chip, receives a single private card, and has the same set of action options. 
In \textit{Leduc Poker}, the game unfolds over two betting rounds. 
During the first round, players have an opportunity to bet two chips, followed by a chance to bet four chips in the second round. 
After the first round, one community card is revealed.
If a player’s private card is paired with the community card, that player wins the game; otherwise, the player holding the highest private card wins the game.

In Flop Hold'em Poker (\textit{\textbf{FHP}}), both players (P1 and P2) start each hand with 20,000 chips, dealt two private cards from a standard 52-card deck.
P1 initially places 100 chips to the pot, followed by P2 adding 50 chips.
P2 starts the first round of betting.
Then players alternate in choosing to fold, call, check or raise.
A maximum of three raises of 100 chips is allowed.
A round ends when a player calls if both players have acted.
After the first round, three community cards are dealt face up for all players to observe, and P1 starts a similar round of betting.
Unless a player has folded, the player with the best five-card poker hand, constructed from their two private cards and the three community cards, wins the pot.
In the case of a tie, the pot is split evenly.

\textbf{\textit{Liar's Dice (x)}} ($x=5,6$)~\cite{lisy2015online} is a dice game where each player gets an $x$-sided dice and a concealment cup. At the beginning of the game, each player rolls their dice under their cup, inspecting the outcome privately. The first player then begins bidding of the form $p$-$q$, announcing that there are at least $p$ dices with the number of $q$ under all the cups. The highest dice number $x$ can be treated as any number. Then players take turns to take action: 1) bidding of the form $p$-$q$, $p$ or $q$ must be greater than the previous player's bidding, 2) calling `Liar', ending the game immediately and revealing all the dices. If the last bid is not satisfied, the player calling `Liar' wins the game. The winner's utility is 1 and the loser -1.

\textbf{\textit{GoofspielImp (x)}} ($x=5,6$)~\cite{ross1971goofspiel} is a bidding card game. At the beginning of the game, each player receives $x$ cards numbered $1 \ldots x$, and there is a shuffled point card deck containing cards numbered $1 \ldots x$. The game proceeds in $x$ rounds. In each round, players select a card from their hand to make a sealed bid for the top revealed point card. When both players have chosen their cards, they show their cards simultaneously. The player who makes the highest bid wins the point card. If the bids are equal, the point card will be discarded. After $x$ rounds, the player with the most point cards wins the game. The winner's utility is 1 and the loser -1. We use a fixed deck of decreasing points. Besides, we use an imperfect information variant where players are only told whether they have won or lost the bid, but not what the other player played.

\textbf{\textit{Battleship (x)}} ($x=2,3$)~\cite{farina2019correlation} is a classic board game where players secretly place a ship on their separate grids of size $2 \times x$ at the start of the game. 
Each ship is $1 \times 2$ in size and has a value of 2. 
Players take turns shooting at their opponent's ship, and the ship that has been hit at all its cells is considered sunk.
The game ends when one player's ship is sunk or when each player has completed three shots. 
The utility for each player is calculated as the sum of the values of the opponent's sunk ship minus the sum of the values of their own lost ship.

We measure the sizes of the games in many dimensions and report the results in Table~\ref{tab:game}.
In the table, \emph{\#Histories} measures the number of histories in the game tree.
\emph{\#Infosets} measures the number of information sets in the game tree.
\emph{\#Terminal histories} measures the number of terminal histories in the game tree.
\emph{Depth} measures the depth of the game tree, \ie, the maximum number of actions in one history.
\emph{Max size of infosets} measures the maximum number of histories that belong to the same information set.
We use the method in~\cite{johanson2013measuring} to compute the sizes of \textit{FHP}.

\begin{table}[h]
	\centering
	\begin{tabular}{c|rrrrr}
			\toprule 
            \textbf{Game}                &  \textbf{\#Histories}   & \textbf{\#Infosets} & \textbf{\#Terminal histories} & \textbf{Depth}        & \textbf{Max size of infosets} \\ \midrule
            Kuhn Poker          &  $58$           & $12$           & $30$           & $6$            & $2$            \\ 
            Leduc Poker         &  $9,457$         & $936$          & $5,520$         & $12$           & $5$            \\ 
            Liar's Dice (5)     &  $51,181$        & $5,120$         & $25,575$        & $14$           & $5$            \\ 
            Liar's Dice (6)     &  $294,883$       & $24,576$        & $147,420$        & $16$           & $6$            \\ 
            GoofspielImp (5)    &  $26,931$        & $2,124$         & $14,400$        & $9$            & $46$           \\
            GoofspielImp (6)    &  $969,523$       & $34,482$        & $518,400$        & $11$            & $230$           \\ 
            Battleship (2)      &  $10,069$        & $3,286$         & $5,568$         & $9$            & $4$            \\ 
            Battleship (3)      &  $732,607$       & $81,027$        & $552,132$       & $9$            & $7$            \\       FHP      & $4.327*10^{12}$     &  $1.455*10^{9}$       &  $2.753*10^{12}$   &  $14$          &  $1,326$  \\ 
			\bottomrule
        \end{tabular}
    \caption{Sizes of the games.}
    \label{tab:game}
\end{table}

\section{Hyperparameters}
\label{app:hypers}

Experiments were conducted on a Linux server with an AMD EPYC 7702 64-Core CPU and 256 GB RAM. Hyperparameter conﬁgurations for each algorithm can be also found in the \texttt{configs} folder in the source code. 

\subsection{Hyperparameters for Neural CFR Variants}

The hyperparameters of OSDeepCFR are adopted from the OpenSpiel reproduction report~\cite{walton2021multi} and are summarized in Table~\ref{tab:neural_cfr}. Additionally, we increase the number of traversals to 10,000, as this adjustment was observed to enhance performance in our experiments. For VR-DeepDCFR+ and VR-DeepPDCFR+, We perform a grid search over $\gamma \in [1, 2, 3, 4, 5]$ and $\alpha \in [2, 2.3, 2.5, 2.7, 3]$, and select the hyperparameters that yields the lowest final exploitability.

\begin{table*}[h]
	\centering
    \begin{tabular}{lcccc}
        \toprule
        \textbf{Hyperparameter Name}     & \textbf{OS-DeepCFR} & \textbf{DREAM} & \textbf{VR-DeepDCFR+} & \textbf{VR-DeepPDCFR+} \\ \midrule
        
        num\_episodes&  10,000,000& 10,000,000& 10,000,000& 10,000,000\\
        advantage\_buffer\_size&  1,000,000&  1,000,000&  1,000,000&  1,000,000\\
        ave\_policy\_buffer\_size&  1,000,000&  1,000,000&  1,000,000&  1,000,000\\
        history\_value\_buffer\_size&  / &  1,000,000&  1,000,000&  1,000,000\\
        learning\_rate&  0.001&  0.001&  0.001&  0.001\\
        num\_traversals&  10,000&  10,000&  10,000&  10,000\\
        advantage\_network\_train\_steps&  750&  750&  750&  750\\
        advantage\_network\_batch\_size&  2,048&  2,048&  2,048&  2,048\\
        ave\_policy\_network\_train\_steps&  5,000&  5,000&  5,000&  5,000\\
        ave\_policy\_batch\_size&  2,048&  2,048&  2,048&  2,048\\
        history\_value\_network\_train\_steps&  /&  10,000&  10,000&  10,000\\
        history\_value\_batch\_size&  /&  2,048&  2,048&  2,048\\
        reinitialize\_advantage\_networks&  True&  True&  False&  False\\
        reinitialize\_imm\_regret\_networks& / &/ &/ &True\\
        num\_layers&  3&  3&  3&  3\\
        num\_hiddens&  64&  64&  64&  64\\
        linear\_weighted&  True & True& /& /\\
        use\_regret\_matching\_argmax&  True&  True&  True&  True\\
        epsilon&  0.6&  0.6&  0.6&  0.6\\
        alpha & / & / & 2 & 2.3\\
        gamma & / & / & 2 & 2 \\
        \bottomrule

	\end{tabular}
    \caption{Hyperparameters for Neural CFR Variants.}
    \label{tab:neural_cfr}
\end{table*}

\subsection{Hyperparameters for NFSP, QPG, and RPG}

The hyperparameters of NFSP are sourced from the OpenSpiel reproduction report~\cite{walton2021multi}. The hyperparameters of QPG and RPG are sourced by~\cite{farina2021model}. All hyperparameters are summarized in Table~\ref{tab:nfsp}.

\begin{table*}[h]
	\centering
    \begin{tabular}{lccc}
        \toprule
        \textbf{Hyperparameter Name}     & \textbf{NFSP} & \textbf{QPG} & \textbf{RPG}\\ \midrule
        
        num\_train\_episodes& 10,000,000& 10,000,000 & 10,000,000 \\
        num\_hidden& 64 & 64 & 64\\
        num\_layers& 3 & 3 & 3\\
        replay\_buffer\_capacity& 1,000,000 & / & /\\
        reservoir\_buffer\_capacity& 500,000& / & /\\
        min\_buffer\_size\_to\_learn& 1,000& / & /\\
        anticipatory\_param& 0.1& / & /\\
        batch\_size& 128& 64 & 128\\
        learn\_every& 128& / & /\\
        rl\_learning\_rate& 0.01 & 0.01 & 0.01\\
        sl\_learning\_rate& 0.01\ & 0.05 & 0.05\\
        update\_target\_network\_every& 19,200& / & /\\
        discount\_factor& 1.0& / & /\\
        epsilon\_decay\_duration& 10,000,000& / & /\\
        epsilon\_start& 0.06& / & /\\
        epsilon\_end& 0.001& / & /\\
        num\_critic\_before\_pi&/& 128 & 64\\
        \bottomrule

	\end{tabular}
    \caption{Hyperparameters for NFSP, QPG, and RPG.}
    \label{tab:nfsp}
\end{table*}

\section{Description of the Rule-based Agents in FHP}
\label{app:rule_agents}
We use five rule-based agents with different styles in \textit{FHP}~\citep{li2018opponent}. At every decision point, these agents compute the expected win rate under the assumption that opponents' hole cards and future public cards follow a uniform distribution, and then take actions according to their respective predefined decision rules. Table~\ref{tab:agents} summarizes the specific decision rules for each agent.

\begin{table}[h]
	\centering
	\begin{tabular}{ll}
			\toprule 
            Agent Name               &  Decision Rules    \\ \midrule
            Candid Statistician &  Raises with most good hands, folds with most weak hand, and calls with marginal hands. \\
            Loose Aggressive & Raises aggressively with a wide range of hands.\\
            Looss Passive & Calls with most hands, folds with weak hands, and rarely raises.\\
            Tight Passive & Calls with good hands, folds with most hands, and rarely raises.\\
            Tight Aggressive & Raises with most good hands, call with marginal hands, fold otherwise, and occasionally bluffs.  \\
			\bottomrule
        \end{tabular}
    \caption{Decision rules of the rule-based agents in FHP.}
    \label{tab:agents}
\end{table}

\section{Ablation Studies}
\label{app:ablation}
This section conducts a detailed ablation study to investigate how the proposed algorithm improves learning efficiency and overall performance. The proposed algorithms consist of three key components: bootstrapped cumulative advantages  estimation, approximating advanced CFR variants, and baseline-based variance reduction. To analyze the contribution of each component, we take VR-DeepPDCFR+ as the base method and remove each module individually, evaluating the resulting variants across four imperfect-information games.

\subsection{Ablation Study of Fitting Cumulative Advantages}
In each iteration, VR-DeepPDCFR+ bootstraps cumulative advantages using a deep neural network instead of cumulative counterfactual regrets. To verify the effectiveness of this design, we conduct an ablation study and analyze its impact. Specifically, VR-DeepPDCFR+ adds the baseline-enhanced sampled advantage $\bar{r}_i^t(I, a\mid z)$ to the advantage buffer at each iteration. In contrast, VR-DeepPDCFR+ w/o adv introduces the sampling weight $\frac{\pi_{-i}^{\sigma^t}(z[I])}{\pi^{\xi^t}(z[I])}$ to compute the baseline-enhanced sampled counterfactual regret, based on the same sampled value. Notably, the expectation of this quantity remains consistent with that of the sampled counterfactual regret $\hat{r}_i^t(I, a \mid z)$, ensuring the estimator remains unbiased despite the use of a baseline function. As shown in Figure~\ref{fig:deepcfr_abl_no_adv}, VR-DeepPDCFR+, which fits cumulative advantages, outperforms its counterpart VR-DeepPDCFR+ w/o adv, especially as the game complexity increases. In more complex games, opponent reach probabilities $\pi_{-i}^{\sigma^t}(z[I])$ tend to vary more significantly across information sets as the game progresses, making it harder for neural networks to fit cumulative counterfactual regrets accurately. This leads to reduced overall performance. These results suggest that fitting cumulative advantages facilitates more efficient learning and improves the overall performance of the algorithm.
\begin{figure*}[h]
    \centering
    \includegraphics[width=1\linewidth]{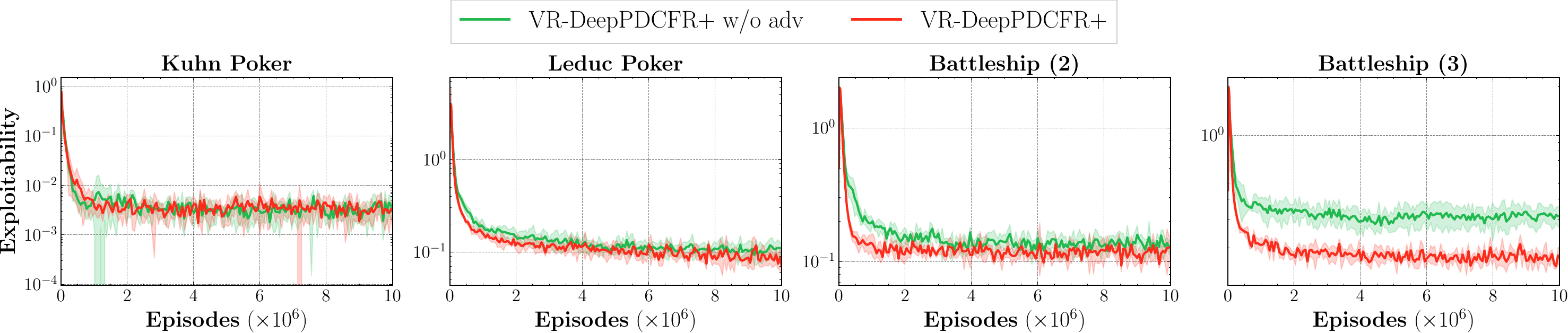}
    \caption{
        Ablation study of fitting cumulative advantages.
    }
    \label{fig:deepcfr_abl_no_adv}
\end{figure*}
\subsection{Ablation Study of Approximating Advanced CFR Variant}
The central motivation of this work is to approximate the behavior of advanced CFR variants using deep neural networks. To evaluate whether the performance gains of VR-DeepPDCFR+ stem from its approximation of PDCFR+, we modify the loss functions of the cumulative advantage and average strategy networks to instead approximate the behavior of Naive CFR or LinearCFR. For VR-DeepLinearCFR, the loss function of the cumulative advantage network is given by:
\begin{align*}
\mathcal{L}(\theta^t_i) = \mathbb{E}_{(I, \check{r} )\sim \mathcal{B}_{V, i}}\left[\sum_{a \in \mathcal{A}(I)}\left(R(I, a \mid \theta^{t-1}_{i} )
\frac{t-1}{t} +\check{r}(I, a)-R(I, a\mid\theta_i^t)\right)^2\right],
\end{align*}
ensuring a linearly weighted update of cumulative advantages. The loss function of the average strategy network adopts $\gamma=1$ to yield a linearly weighted average strategy. In both variants, the sampled advantage is replaced by the baseline-enhanced sampled advantage. As shown in Figure~\ref{fig:deepcfr_abl_cfr_variants}, VR-DeepPDCFR+ achieves the fastest convergence. Furthermore, VR-DeepLinearCFR outperforms VR-DeepCFR, a trend consistent with their corresponding tabular counterparts. These results demonstrate that approximating advanced CFR variants plays a crucial role in enhancing algorithm performance.
\begin{figure*}[h]
    \centering
    \includegraphics[width=1\linewidth]{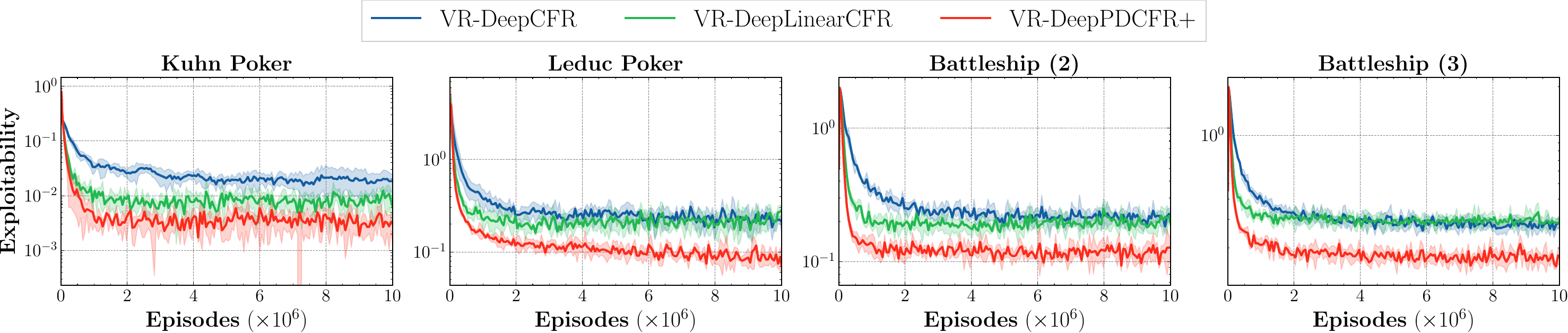}
    \caption{
        Ablation study of approximating advanced CFR variants.
    }
    \label{fig:deepcfr_abl_cfr_variants}
\end{figure*}
\subsection{Ablation of Variance Reduction}
Lastly, we study the impact of the baseline-based variance reduction method introduced in VR-DeepPDCFR+. While DeepPDCFR+ adds the sampled advantage $\check{r}_i^t(I, a\mid z)$ to the advantage buffer at each iteration, VR-DeepPDCFR+ uses the baseline-enhanced sampled advantage $\bar{r}_i^t(I, a\mid z)$. As shown in Figure~\ref{fig:deepcfr_abl_no_baseline}, VR-DeepPDCFR+ converges faster than DeepPDCFR+. This improvement may result from the reduced variance of sampled advantages due to the baseline, which makes each strategy update less susceptible to randomness and allows the average strategy to approach the Nash equilibrium more efficiently.
\begin{figure*}[h]
    \centering
    \includegraphics[width=1\linewidth]{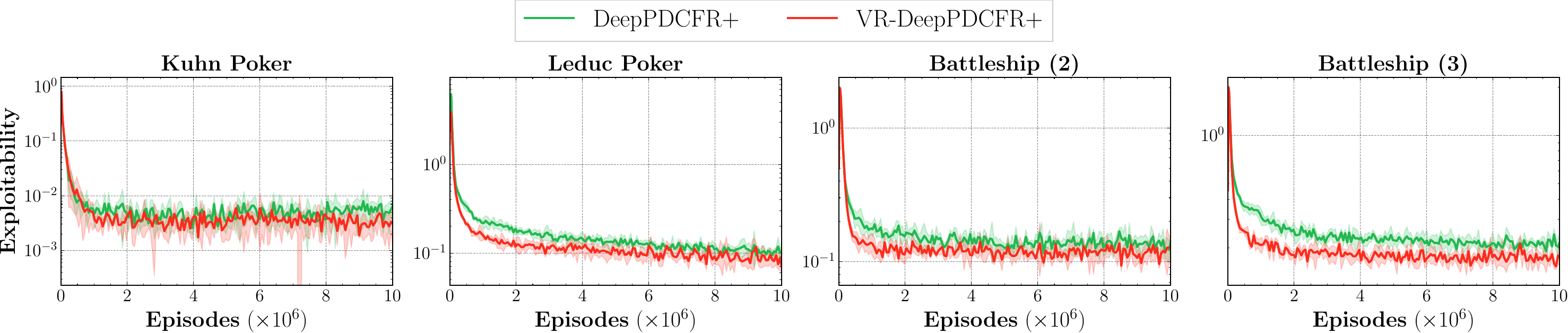}
    \caption{
        Ablation study of variance reduction.
    }
    \label{fig:deepcfr_abl_no_baseline}
\end{figure*}

\end{document}